\documentclass[letterpaper]{article} 
\usepackage{aaai2026}  
\usepackage{times}  
\usepackage{helvet}  
\usepackage{courier}  
\usepackage[hyphens]{url}  
\usepackage{graphicx} 
\usepackage{natbib}  
\usepackage{caption} 
\usepackage{algorithm}
\usepackage{algorithmic}

\usepackage{newfloat}
\usepackage{listings}
\DeclareCaptionStyle{ruled}{labelfont=normalfont,labelsep=colon,strut=off} 
\floatstyle{ruled}
\newfloat{listing}{tb}{lst}{}
\floatname{listing}{Listing}
\usepackage{amsmath}
\usepackage{amssymb}
\usepackage{amsfonts}
\usepackage{amsthm}
\usepackage{xcolor}

\newtheorem{lemma}{Lemma}
\newtheorem{theorem}{Theorem}
\newtheorem{proposition}{Proposition}
 
\newtheorem*{restatedtheorem}{Theorem}

\usepackage{ifthen}

\newboolean{includesupplement}
\setboolean{includesupplement}{true}

\title{Online Robust Planning under Model Uncertainty: \\A Sample-Based Approach}

\title{Online Robust Planning under Model Uncertainty: \\A Sample-Based Approach}

\author{
Tamir Shazman\textsuperscript{\rm 1}, 
Idan Lev-Yehudi\textsuperscript{\rm 2}, 
Ron Benchetrit\textsuperscript{\rm 3}, 
Vadim Indelman\textsuperscript{\rm 4, 1}
}
\affiliations{
\textsuperscript{\rm 1}Faculty of Data and Decision Sciences \\
\textsuperscript{\rm 2}Technion Autonomous Systems Program (TASP) \\
\textsuperscript{\rm 3}Faculty of Computer Science \\
\textsuperscript{\rm 4}Stephen B. Klein Faculty of Aerospace Engineering \\
Technion -- Israel Institute of Technology, Haifa 32000, Israel \\
tmyr@campus.technion.ac.il, vadim.indelman@technion.ac.il
}

\begin{document}

\maketitle

\begin{abstract}
Online planning in Markov Decision Processes (MDPs) enables agents to make sequential decisions by simulating future trajectories from the current state, making it well-suited for large-scale or dynamic environments. Sample-based methods such as Sparse Sampling and Monte Carlo Tree Search (MCTS) are widely adopted for their ability to approximate optimal actions using a generative model. However, in practical settings, the generative model is often learned from limited data, introducing approximation errors that can degrade performance or lead to unsafe behaviors. To address these challenges, Robust MDPs (RMDPs) offer a principled framework for planning under model uncertainty, yet existing approaches are typically computationally intensive and not suited for real-time use. In this work, we introduce \emph{Robust Sparse Sampling} (RSS), the first online planning algorithm for RMDPs with finite-sample theoretical performance guarantees. Unlike Sparse Sampling, which estimates the nominal value function, RSS computes a robust value function by leveraging the efficiency and theoretical properties of Sample Average Approximation (SAA), enabling tractable robust policy computation in online settings. RSS is applicable to infinite or continuous state spaces, and its sample and computational complexities are independent of the state space size. We provide theoretical performance guarantees and empirically show that RSS outperforms standard Sparse Sampling in environments with uncertain dynamics.
\end{abstract}

\bigskip

\section{Introduction}
Markov Decision Processes (MDPs) provide a mathematical framework for modeling sequential decision-making under uncertainty, where an agent interacts with a stochastic environment to maximize cumulative expected rewards.
Exact solutions to MDPs are often computationally infeasible, as they've been shown to be P-complete \cite{papadimitriou1987complexity}, and practical methods often resort to approximate solutions \cite{littman1995complexity}. 

\emph{Online methods} try to circumvent the complexity of computing a policy by planning online only for the current state \cite{koenig2001agent,ross2008online}, making it particularly suitable for large or dynamic environments where computing optimal actions for all states in advance is infeasible. Among the most popular online planning methods are sample-based algorithms like Sparse Sampling~\cite{kearns2002sparse} and Monte Carlo Tree Search (MCTS)~\cite{coulom2006efficient}, which approximate near-optimal decisions using limited computation at runtime.

Sparse Sampling has historical, theoretical and algorithmic significance: being the first algorithm to provide finite-time guarantees for online planning, and computational complexity scaling only by the planning horizon and approximation budget, rather than state-space size. It has inspired many practical tree-based online planning algorithms like MCTS \cite{kocsis2006bandit,browne2012survey,silver2016mastering}, popular algorithms for partially observable settings like POMCP \cite{silver2010monte} and DESPOT \cite{somani2013despot}, and has recently been extended to theoretical guarantees of particle-belief approximations in POMDP planning \cite{lim2019sparse,lim2023optimality}.

A major limitation of Sparse Sampling, MCTS and existing online planning methods is that they typically assume access to a \emph{generative model}, i.e. a simulator that provides samples of next states and rewards. In practice, however, such models are often estimated from data and may introduce approximation errors. If these discrepancies are ignored, they can lead to poor or even unsafe decision-making~\cite{mannor2007bias}.

\emph{Robust Markov Decision Processes} (RMDPs) offer a theoretical framework to address this issue by explicitly modeling uncertainty in the transition dynamics~\cite{iyengar2005robust, nilim2005robust}. RMDPs define sets of plausible models and optimize for the worst-case within these sets, thereby guaranteeing performance robustness. However, solving RMDPs typically involves a computationally demanding min-max optimization over both policies and model perturbations, making them difficult to apply in online or large-scale settings.

To enhance scalability, various approaches have been proposed within the robust reinforcement learning (RL) community. Some methods utilize robust variants of function approximation~\cite{tamar2014scaling}, while others introduce sample-based algorithms that learn robust policies by interacting with an environment affected by model uncertainty~\cite{wang2021online, panaganti2022robustreinforcementlearningusing, panaganti2022sample, dong2022online}. Although these techniques show promise, they are generally not tailored for online planning, as they aim to learn global policies across the entire state space rather than allocating computational effort to the specific decision at hand.
In contrast to robust reinforcement learning, few works address the challenges of robust online planning, being limited to parametric uncertainty structures or deterministic MDPs \cite{sharma2019robust, kohankhaki2024monte}. This highlights the need for robust online planning methods that are both general-purpose and theoretically grounded.

To address the challenge of online planning under model uncertainty, we adopt the RMDP framework to formalize robustness and introduce a new sample-based planning algorithm: \emph{Robust Sparse Sampling (RSS)}. RSS extends Sparse Sampling to explicitly handle model uncertainty. To enable efficient robust decision-making in online settings, it leverages the theoretical properties and computational efficiency of the Sample Average Approximation framework \cite{shapiro2021lectures}. We also establish theoretical performance guarantees.

\subsection{Contributions}


This work addresses \emph{online planning under model uncertainty} by formulating a sample-based robust planner and establishing its theoretical and empirical merits. Our main contributions are:

\begin{itemize}
\item \textbf{Algorithmic novelty.}
We propose \emph{Robust Sparse Sampling} (RSS), which, to the best of our knowledge, is the first sample-based online planning algorithm that directly addresses robust MDPs while providing \emph{finite-sample performance guarantees}.
Notably, the complexity of RSS is independent of the size of the state space, making it suitable for environments with infinite or continuous state spaces.

\item \textbf{Theoretical guarantees.}
By leveraging the convergence theory of Sample Average Approximation technique, we derive an error bound between the value of the policy induced by RSS and the true optimal robust value function. This bound can be made arbitrarily small by appropriately setting the planning parameters.

\item \textbf{Empirical validation.}
Experiments on two benchmark domains demonstrate that RSS substantially reduces catastrophic failures and achieves higher empirical returns than classical Sparse Sampling when the transition dynamics are misspecified.
\end{itemize}

\subsection{Related Work}

Robustness in \emph{online MDP planning} has been explored in only a couple of recent studies. \citet{sharma2019robust} introduced Robust Adaptive Monte Carlo Planning (RAMCP), which embeds Monte Carlo Tree Search in the Bayes-adaptive framework. That framework requires a prior over the transition model, and misspecifying this prior can harm performance; RAMCP seeks to hedge against such misspecification by computing a policy that is robust to prior errors. However, RAMCP still assumes that transition uncertainty follows a specific parametric form, which limits its applicability in settings where the dynamics are non-parametric or deviate from that model.

\citet{kohankhaki2024monte} introduced Uncertainty Adapted MCTS (UA-MCTS), an MCTS variant for deterministic MDPs that adjusts node selection based on estimated transition uncertainty. Although UA-MCTS demonstrates strong empirical performance in deterministic settings, it lacks formal robustness guarantees and does not extend naturally to stochastic environments, limiting its general applicability.

\section{Preliminaries}
%

\subsection{Robust Markov Decision Process (RMDP)}

We consider a Markov Decision Process (MDP) defined as $\mathcal{M} = (\mathcal{S}, \mathcal{A}, r, P, \gamma)$. The (possibly infinite) state space is $\mathcal{S}$. We assume the action space $\mathcal{A}$ is finite. We assume a bounded reward function $r: \mathcal{S} \times \mathcal{A} \rightarrow [0,1]$, yet our analysis can be trivially extended to any time-dependent bounded reward. For the transition kernel $P$, $P_{s,a}(s')$ denotes the probability of transitioning to state $s'$ given state $s$ and action $a$. $\gamma \in [0,1)$ is the discount factor.

During planning, the agent has access only to an approximate generative model of the transition kernel $P^o$, which is an estimate of the true transition model $P$.
We assume that there exists a state-action dependent bound between the true and approximate transition kernels of the form:
\begin{align}
\label{eq:uncertainty_assumption}
\forall (s,a) \in \mathcal{S} \times \mathcal{A}, \quad D(P_{s,a}, P^o_{s,a})  \leq \rho,
\end{align}
where $\rho \in [0,1]$, and $D(\cdot, \cdot)$ is a distance metric between two probability distributions. $\rho$ quantifies the maximum allowable deviation between the true transition model $P_{s,a}$ and the estimated model $P^o_{s,a}$. Higher values of $\rho$ indicate greater uncertainty, with $\rho = 0$ meaning perfect model accuracy. This uncertainty bound can be estimated from statistical confidence intervals \cite{berend2012convergence} or explicitly defined based on domain-specific knowledge. In this work, we focus on $D(\cdot, \cdot)$ being the Total Variation (TV) distance, i.e., $D(P_{s,a}, P^o_{s,a}) = \frac{1}{2} \lVert P_{s,a} - P^o_{s,a} \rVert_1$. 

Planning directly with the empirical model $P^o$ can lead to suboptimal or unsafe policies \cite{mannor2007bias}.
To guard against model error, we adopt the Robust MDP framework (RMDP). Instead of a single transition kernel, RMDP consider an uncertainty set of transition kernels. Adopting the common rectangularity assumption \cite{iyengar2005robust, nilim2005robust}, according to which the uncertainty in the transition kernels is independent for each state-action pair, we define the uncertainty set as:
\begin{align}
    &\mathcal{P} = \bigotimes_{(s,a) \in \mathcal{S} \times \mathcal{A}} \mathcal{P}_{s,a}, \label{eq:UncertaintySet}\\
    &\mathcal{P}_{s,a} = \left\{ P_{s,a} \in \Delta(\mathcal{S}) : D(P_{s,a}, P^o_{s,a}) \leq \rho \right\},
\end{align}
where $\Delta(\mathcal{S})$ is the set of probability distributions over $\mathcal{S}$.
For a fixed model $P^\prime$ and policy $\pi$, the (non-robust) value function is
\begin{align}
V^{\pi,P^\prime}(s) &= \mathbb{E}_{P^\prime,\pi} \Big[ \sum_{t=0}^{\infty} \gamma^t r(s_t, a_t) \mid s_0 = s, a_t = \pi(s_t) \Big].
\end{align}
The \emph{robust} value function takes the worst case model in $\mathcal{P}$:
\begin{align}
    V^\pi(s) &= \min_{P' \in \mathcal{P}} V^{\pi,P'}(s),
\end{align}
and our planning objective is to find a policy that maximizes this worst-case return, i.e. $V^*(s) = \max_{\pi} V^\pi(s)$ where $\pi^* \in \arg\max_{\pi} V^\pi(s)$.
We denote the corresponding robust action-value function by $Q^*$.
A deterministic robust optimal policy is known to exist \cite{iyengar2005robust}, and its value function satisfies the robust Bellman equation:
\begin{align}
\label{eq:Robust_bellman_equation}
V^*(s) \!=\! \max_{a \in \mathcal{A}} \Big[ r(s,a) \!+\! \gamma \! \min_{P_{s,a} \in \mathcal{P}_{s,a}} \! \!\!\mathbb{E}_{s' \sim P_{s,a}} \big[ V^*(s') \big] \Big].
\end{align}
This formulation guarantees that the robust value serves as a lower-bound for the true value, providing explicit protection against transition-model misspecification.

\subsection{Robust Action-Value Function Dual Form}

Computing a robust policy under an imperfect transition model using online, sample-based methods is challenging, due to the infinite number of possible transition distributions within the uncertainty set \eqref{eq:UncertaintySet}. This makes a direct optimization of the robust Bellman's equation \eqref{eq:Robust_bellman_equation} intractable.
In their recent work, \citet{panaganti2022robustreinforcementlearningusing} show that the dual form of the optimal robust action-value function admits the following closed-form expression:
\begin{equation}
\begin{aligned}
&Q^*(s,a) = r(s,a) - \\
&\gamma \min_{\eta \in \left[0, \frac{2}{\rho(1 - \gamma)}\right]} \Big(\mathbb{E}_{s' \sim P^o_{s,a}} \left[(\eta - V^*(s'))_+\right] - \eta +\\
& \rho \left(\eta - \inf_{s''} V^*(s'')\right)\Big), \label{eq:dual_formulation}
\end{aligned}
\end{equation}
where $[x]_+ \triangleq \max\{0,x\}$. The dual variable $\eta$ serves as a Lagrange multiplier, balancing the trade-off between the expected value and the worst-case value.

However, estimating the infimum of the robust value function $V^*(s'')$ over all states $s''$ is generally intractable, and particularly problematic in large or continuous state spaces, where computing the infimum term is computationally prohibitive.
To simplify the dual formulation, we assume the existence of a fail-state, stated in the following assumption.
\paragraph{Assumption (Fail-State).} 
There exists a state $s_f \in \mathcal{S}$ such that $r(s_f, a) = 0$ and $P'_{s_f, a}(s_f) = 1$ for all actions $a \in \mathcal{A}$ and all transition probabilities $P' \in \mathcal{P}$. This implies $V^*(s_f) = 0$, and hence $\inf_{s''} V^*(s'') = 0$.
Under this assumption, equation \eqref{eq:dual_formulation} simplifies to:
\begin{equation}
\label{eq:dual_reformulation}
\begin{aligned}
&Q^*(s,a) = r(s,a) - \\
&\gamma \min_{\eta \in \left[0, \frac{2}{\rho(1 - \gamma)}\right]} \Big(\mathbb{E}_{s' \sim P^o_{s,a}} \left[(\eta - V^*(s'))_+\right] - \eta(1 - \rho)\Big).
\end{aligned}
\end{equation}

To simplify the notation in the remainder of the paper, we define for each state-action pair $(s, a)$ the function:
\begin{equation}
F^{\rho}_{s,a}(\eta) \triangleq \mathbb{E}_{s' \sim P^o_{s,a}} \left[(\eta - V^*(s'))_+\right] - \eta(1 - \rho).
\label{eq:dual_F_function}
\end{equation}
Hence, we can rewrite the dual action-value in equation~\eqref{eq:dual_reformulation}:
\begin{equation}
\label{eq:dual_reformulation_short}
Q^*(s,a) = r(s,a) - \gamma \min_{\eta \in \left[0, \frac{2}{\rho(1 - \gamma)}\right]} F^{\rho}_{s,a}(\eta).
\end{equation}

\subsection{Sparse Sampling (SS)}
Sparse Sampling (SS) \cite{kearns2002sparse} is a model-based online planning algorithm assuming a known transition kernel $P$, that approximates the optimal action-value function $Q^{*,P}$ with high probability by constructing a stochastic lookahead tree of finite depth $H$.  
It operates by building a recursive search tree. At each node corresponding to a state $s$, the algorithm explores each action $a \in \mathcal{A}$ by drawing $C$ independent next-state samples from $P^{s,a}(\cdot)$. For each sampled successor state $s'$, the process recursively continues until the maximum depth $H$ is reached. 
The recursive computation of the action-value function at depth $d$ proceeds as follows:
\begin{align}
    \label{eq:sparse_sampling_recursion}
    &\hat{Q}^{P}_d(s,a) = r(s,a) + \gamma \cdot \frac{1}{C} \sum_{i=1}^{C} \hat{V}^{P}_{d-1}(s'_i), \quad s'_i \sim P_{s,a}, \nonumber \\
    &\hat{V}^{P}_{d-1}(s) = \max_{a \in \mathcal{A}} \hat{Q}^{P}_{d-1}(s,a), \\
    &\hat{V}^{P}_{0}(s) = \tilde{V}^{P}_{\theta}(s), \quad \forall s \in \mathcal{S} \text{ (leaf terminal value)}. \nonumber
\end{align}
Here, $\tilde{V}^{P}_{\theta}(s)$ denotes a terminal value estimator for $V^*(s')$, which may be a learned function or set to zero. In this work, unless stated otherwise, we assume $\tilde{V}^{P}_{\theta}(s) = 0$.

Sparse Sampling provides theoretical guarantees on the gap between the nominal value of the policy it computes and the optimal value function. This difference can be made arbitrarily small by choosing a sufficiently large number of samples $C$ and planning depth $H$.

\subsection{Sample Average Approximation (SAA)} 
Stochastic programming~\cite{haneveld2020stochastic} addresses optimization problems under uncertainty, where the objective function involves an expectation over a random variable. The general formulation is given by:
\begin{equation}
\begin{aligned}
\min_{x \in \mathbb{X}} &\quad F(x), \\
\text{where} \quad F(x) &\triangleq \mathbb{E}_{y \sim P_y} \left[ f(y, x) \right],
\label{eq:stochastic_programming_general}
\end{aligned}
\end{equation}
Here, $P_y$ denotes a probability distribution over random variables $y$, $\mathbb{X} \subseteq \mathbb{R}$ is the feasible domain, and $f(y, x)$ is a real-valued function depending on both the uncertain variable $y$ and the decision variable $x$. 

Computing the expectation $\mathbb{E}_{y \sim P_y} \left[ f(y, x) \right]$ exactly can be challenging, especially when the distribution $P_y$ is high-dimensional or analytically intractable. Sample Average Approximation (SAA)~\cite{shapiro2021lectures} replaces the expectation with an empirical average based on a finite number of samples drawn from $P_y$. Given $C$ i.i.d.\ samples $\{y_i\}_{i=1}^C$ from $P_y$, the empirical approximation of the objective becomes:
\begin{equation}
\hat{F}(x) = \frac{1}{C} \sum_{i=1}^{C} f(y_i, x).
\end{equation}
SAA is widely used across domains such as operations research, finance, and machine learning to address optimization problems under uncertainty~\cite{verweij2003sample, bertsimas2018robust, burroni2023sample, shapiro2025risk}. Its convergence properties are well-established~\cite{sinha2024multilevel}; under suitable regularity conditions on the function $f(y, x)$ and the feasible set $\mathbb{X}$, the solution of the empirical problem converges to the true optimum of the original stochastic program as the number of samples increases \cite{shapiro2021lectures}. 



\section{Robust Sparse Sampling (RSS)}

\subsection{Robust Action-Value Estimation via SAA}

Although the simplified dual formulation in Equation~\eqref{eq:dual_reformulation_short} offers valuable theoretical insight, it remains intractable to solve directly when the robust value function \( V^*(\cdot) \) is unknown. Even if \( V^*(\cdot) \) were available, evaluating \( Q^*(s,a) \) would still require solving a stochastic programming problem over the function \( F^{\rho}_{s,a}(\eta) \), which involves an intractable expectation, particularly in large or continuous state spaces.

To address this challenge, we employ the SAA method, replacing \( F^{\rho}_{s,a}(\eta) \) with an empirical estimate \( \hat{F}^{\rho}_{s,a}(\eta) \) based on $C$ samples of the next states \( \{s'_i\}_{i=1}^{C} \sim P^o_{s,a}(\cdot) \). This leads to the following approximate formulation:

\begin{equation}
\label{eq:saa_optimization}
\begin{aligned}
\hat{Q}^*(s,a) &= r(s,a) - \gamma \min_{\eta \in \left[0, \frac{2}{\rho(1 - \gamma)}\right]} \hat{F}^{\rho}_{s,a}(\eta), \quad\text{where} \\
\hat{F}^{\rho}_{s,a}(\eta) &= \frac{1}{C} \sum_{i=1}^{C} (\eta - V^*(s_i'))_+ - \eta(1 - \rho).
\end{aligned}
\end{equation}
The function \( \hat{F}^{\rho}_{s,a}(\eta) \) is piecewise linear and convex in \( \eta \), with non-differentiable breakpoints occurring at the sampled values \( \{V^*(s_i')\}_{i=1}^C \). This structure makes the optimization problem in Equation~\eqref{eq:saa_optimization} efficiently solvable.

\subsection{RSS Algorithm}
The Robust Sparse Sampling (RSS) algorithm, inspired by the Sparse Sampling algorithm, incorporates robustness against model uncertainty while using a finite number of samples online in a recursive manner. Instead of estimating the nominal action-value function $Q^{*, P}(s,a)$, RSS estimates the robust action-value function $Q^{*}(s,a)$. The complete procedure is described in Algorithm~\ref{alg:robust_sparse_sampling}.

Specifically, RSS recursively estimates the robust action-value function at depth $d$ by sampling $C$ successor states from the estimated generative model $P^o_{s,a}$. Then, in contrast to the standard Sparse Sampling, RSS computes the robust action-value function by solving the SAA problem at each depth $d$: 
\begin{equation}
\label{eq:robust_sparse_sampling}
\begin{aligned}
    &\hat{Q}_d(s,a) = r(s,a) - \gamma \min_{\eta \in \left[0, \frac{2}{\rho(1 - \gamma)}\right]} \tilde{F}^{\rho, d}_{s,a}(\eta),&\quad\text{where} \\
&\tilde{F}^{\rho, d}_{s,a}(\eta) = \frac{1}{C} \sum_{i=1}^{C} (\eta - \hat{V}_{d-1}(s'_i))_+ - \eta(1 - \rho),\\
    &\hat{V}_{d-1}(s) = \max_{a \in \mathcal{A}} \hat{Q}_{d-1}(s,a),\\
    &\hat{V}_{0}(s) = 0, \quad \forall s \in \mathcal{S} \text{ (leaf terminal value)}.
\end{aligned}
\end{equation}
where each successor $s^{\prime}_i$ is drawn i.i.d. from the approximate model $P^o_{s,a}(\cdot)$.
The routine is invoked recursively from the current state $s$ and remaining depth $d$. The recursion terminates at $d=0$, where the leaf value is fixed at 0. .

It is important to emphasize that the function \( \hat{F}^{\rho}_{s,a}(\eta) \) in Equation~\eqref{eq:saa_optimization} is a theoretical construct, as it depends on the true robust value function \( V^*(\cdot) \), which is not accessible to the algorithm in practice. In contrast, the RSS algorithm avoids this dependency by estimating robust values recursively. 

At depth \( d \), RSS replaces \( V^*(s') \) with \( \hat{V}_{d-1}(s') \), the estimated robust value of the sampled successor state \( s' \) from the previous depth. This substitution yields an empirical estimate \( \tilde{F}^{\rho, d}_{s,a}(\eta) \) that approximates \( F^{\rho}_{s,a}(\eta) \) without requiring knowledge of the exact robust value function.





\begin{algorithm}[tb]
\caption{Robust Sparse Sampling (RSS)}
\label{alg:robust_sparse_sampling}
\textbf{Input}: Current state $s$, current depth $d$ \\
\textbf{Parameter}: Sample width $C$, planning horizon $H$ computed based on Theorem~\ref{thm:robust-sparse-sampling-guarantee} \\
\textbf{Output}: \parbox[t]{.85\linewidth}{Estimated optimal action and its value}
\begin{algorithmic}[1]
\IF{$d = 0$}
    \STATE \textbf{return} 0
\ENDIF
\FORALL{$a \in \mathcal{A}$}
    \STATE $V_{\text{list}} \gets [\,]$
    \FOR{$i = 1$ to $C$}
        \STATE Sample $s'_i \sim P^o(\cdot \mid s, a)$
        \STATE $(\_, \ \hat{V}_{d-1}(s'_i)) \gets$ \textsc{RSS}$(s'_i, d-1)$
        \STATE Append $\hat{V}_{d-1}(s'_i)$ to $V_{\text{list}}$
    \ENDFOR
\STATE Update $\hat{Q}_d (s,a)$ using Equation~\eqref{eq:robust_sparse_sampling} with $V_{\text{list}}$

\ENDFOR
\STATE \textbf{return} $\arg\max_{a \in \mathcal{A}} \hat{Q}_d (s,a), \max_{a \in \mathcal{A}} \hat{Q}_d (s,a)$
\end{algorithmic}
\end{algorithm}

\section{Theoretical Analysis of RSS}

\subsection{Performance Guarantees}
Our main theoretical result establishes that the value of the policy returned by the RSS algorithm can be made arbitrarily close to the optimal robust value function. Specifically, RSS guarantees the following bound:

\begin{theorem}
\label{thm:robust-sparse-sampling-guarantee}
For any $s \in \mathcal{S}$ and any $\epsilon > 0$, the Robust Sparse Sampling algorithm returns a policy $\pi$ such that:
\begin{align*}
\left|V^{\pi}(s) - V^*(s)\right| \leq \epsilon,
\end{align*}
with the following hyperparameters:
\begin{align*}
&\lambda = \frac{\epsilon}{3}, \ \delta = \lambda (1 - \gamma), \\
&H = \left\lceil \log_{\gamma}(\lambda) \right\rceil, \\
&C = \frac{2}{\lambda^2 \rho^2 (1-\gamma)^2} \cdot\\
&\left( 2H \ln \left(\frac{2 \lvert A \rvert \cdot H}{\lambda^2 \rho^2 (1-\gamma)^2} \right) + \ln \left( \frac{2(8-4\rho)}{\delta \lambda (1-\gamma) \rho}\right)\right).
\end{align*}
\end{theorem}

A detailed proof is provided in Appendix~\ref{appendix:proofs} within the supplementary material. Similar result was originally shown for the nominal (non-robust) setting by~\citet{kearns2002sparse}, where the Sparse Sampling algorithm approximates the optimal value function \( V^{*,P}(s) \). Here, we extend that result to the robust setting.

This extension is non-trivial, as the robust formulation must account for worst-case transitions within an uncertainty set, which do not arise in the nominal case.

\paragraph{Proof Sketch.}
The proof of Theorem~\ref{thm:robust-sparse-sampling-guarantee} follows a structure similar to the original Sparse Sampling analysis~\cite{kearns2002sparse}, but extends it using tools from SAA theory to handle robustness.

First, we show that both  \( F_{s,a}^{\rho}(\eta) \), defined in~\eqref{eq:dual_reformulation_short}, and its empirical counterpart \( \hat{F}_{s,a}^{\rho}(\eta) \), defined in~\eqref{eq:saa_optimization}, are Lipschitz continuous with respect to \( \eta \). This property allows us to apply concentration inequalities from SAA theory~\cite{shapiro2021lectures}, yielding probabilistic bounds between \( F_{s,a}^{\rho}(\eta) \) and \( \hat{F}_{s,a}^{\rho}(\eta) \). Consequently, we obtain bounds on the difference between the true robust action-value function \( Q^*(s,a) \) and the SAA-based estimate \( \hat{Q}^*(s,a) \).
Next, we establish a concentration bound between \( F_{s,a}^{\rho}(\eta) \) and \( \tilde{F}^{\rho,d}_{s,a}(\eta) \), the estimator used by RSS at depth \( d \), as defined in~\eqref{eq:robust_sparse_sampling}. This step uses the bound from the previous stage, combined with an inductive argument over the estimated robust value function \( \hat{V}_{d-1} \) at depth \( d - 1 \). The relationship and differences between \( F_{s,a}^{\rho}(\eta) \), \( \hat{F}_{s,a}^{\rho}(\eta) \), and \( \tilde{F}^{\rho,d}_{s,a}(\eta) \) are illustrated in Figure~\ref{fig:RSS_illustration}.

\begin{figure}[htbp]
    \centering
    \includegraphics[width=1.0\linewidth]{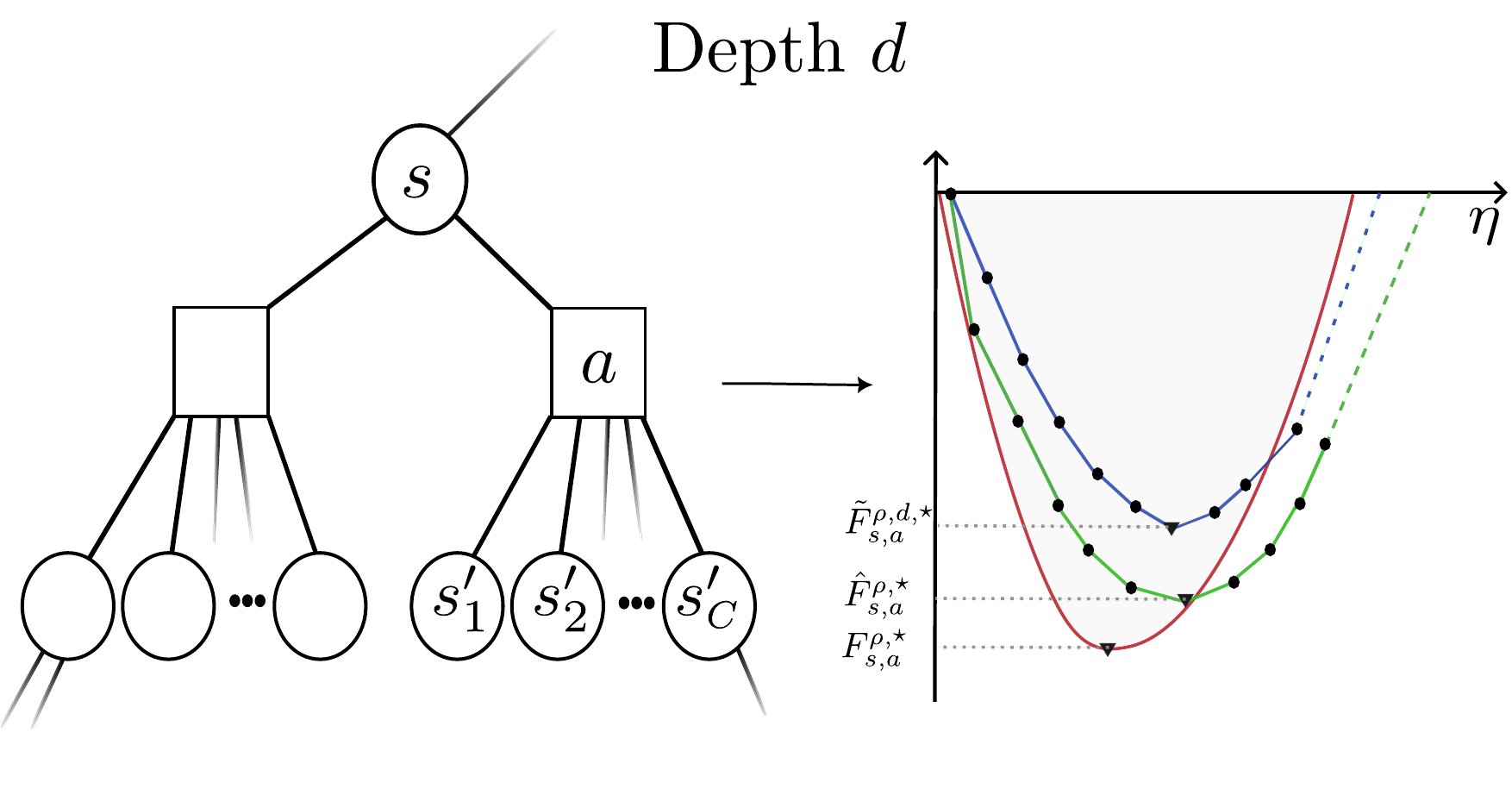} 
    \caption{Illustration of the RSS algorithm. At each depth~$d$, RSS samples~$C$ successor states and recursively estimates their robust values~$\hat{V}_{d-1}(s_i')$. The robust action-value estimate~$\hat{Q}_d(s,a)$ is obtained by minimizing the piecewise-linear convex function~$\tilde{F}^{\rho,d}_{s,a}(\eta)$. The plot displays~$F^{\rho}_{s,a}(\eta)$,~$\hat{F}^{\rho}_{s,a}(\eta)$, and~$\tilde{F}^{\rho,d}_{s,a}(\eta)$ in red, green, and blue, respectively, along with their corresponding minima~$F^{\rho, \star}_{s,a}(\eta)$,~$\hat{F}^{\rho, \star}_{s,a}(\eta)$, and~$\tilde{F}^{\rho,d, \star}_{s,a}(\eta)$. Triangles indicate the minima, and black dots represent the breakpoints. As the number of samples~$C$ increases, both~$\hat{F}^{\rho}_{s,a}(\eta)$ and~$\tilde{F}^{\rho,d}_{s,a}(\eta)$ converge to $F^{\rho}_{s,a}(\eta)$. Since both~$\tilde{F}^{\rho,d}_{s,a}(\eta)$ and~$\hat{F}^{\rho}_{s,a}(\eta)$ are piecewise-linear and convex, their minima can be computed efficiently.}
    \label{fig:RSS_illustration}
\end{figure}
We then apply a union bound over every state-action pair in the search tree, guaranteeing that the concentration inequalities hold simultaneously at all nodes.

Finally, to relate the robust value of the policy \( V^{\pi}(s) \) returned by RSS to the optimal robust value \( V^*(s) \), we generalize a key lemma from~\cite{kearns2002sparse} to the robust setting. This lemma bounds the value gap in terms of the maximum approximation error in the robust action-value function. Combining all steps yields the final bound stated in Theorem~\ref{thm:robust-sparse-sampling-guarantee}.

\subsection{Computational Complexity}
The RSS algorithm builds a lookahead tree of depth \( H \), where each node branches into \( \lvert A \rvert \cdot C \) children—corresponding to \( \lvert A \rvert \) actions and \( C \) sampled next states per action. Therefore, the total number of nodes in the tree is \( (\lvert A \rvert \cdot C)^H \).

At each node, the algorithm performs two main operations:  
(1) sampling \( C \) next states using the generative model, and  
(2) solving the SAA optimization problem defined in Equation~\ref{eq:robust_sparse_sampling}.  
The sampling step incurs a cost of \( O(C) \).  
For the optimization step, the algorithm minimizes a piecewise-linear convex function \( \tilde{F}^{\rho, d}_{s,a}(\eta) \), which has \( C \) breakpoints at the values \( \{\hat{V}_{d-1}(s_i')\}_{i=1}^C \). The minimum is guaranteed to lie at one of the breakpoints or at the boundary points \( \eta = 0 \) and \( \eta = \frac{2}{\rho(1 - \gamma)} \). The optimal solution can thus be found by first sorting the \( C \) breakpoints in \( O(C \log C) \) time, followed by a linear scan to identify the minimizer, resulting in a total per-node complexity of:
\[
O(C \log C).
\]
Multiplying this by the total number of nodes yields the overall computational complexity of RSS:
\[
O\left( (\lvert A \rvert \cdot C \log C)^H \right).
\]
For comparison, the standard Sparse Sampling algorithm has complexity \( O\left((\lvert A \rvert \cdot C)^H\right) \), implying that RSS introduces only an additional logarithmic factor due to the robust optimization step.
Importantly, in most practical applications, sampling successor states from the generative model dominates the computational cost. As a result, the added \( \log C \) factor in RSS is typically negligible in practice and does not significantly affect overall runtime.

\section{Experiments}
\label{sec:experiments}

We evaluate the performance of the proposed RSS algorithm in two benchmark environments: \texttt{FrozenLake} and \texttt{CartPole}, aiming to empirically assess its robustness under model misspecification and compare it to standard Sparse Sampling (SS).

All experiments are conducted in the setting of \emph{online planning with model uncertainty}. The agent computes actions by simulating future trajectories using an inaccurate generative model, differing from the true environment dynamics.

In these environments, uncertainty is present only in certain regions, while others are accurately modeled. This reflects a common real-world scenario in which hazardous or rarely visited states lack sufficient data, resulting in higher model uncertainty, whereas frequently visited safe regions benefit from more reliable transition estimates. To capture this structure, we apply the robust backup update \eqref{eq:robust_sparse_sampling} exclusively in states with uncertainty. In all other states, we use the standard expected backup \eqref{eq:sparse_sampling_recursion}. Full algorithmic details are provided in Algorithm~\ref{alg:robust_sparse_sampling_chaning_rho} within the supplementary material.

This selective use of robust backups preserves the theoretical guarantees established in Theorem~\ref{thm:robust-sparse-sampling-guarantee}, while avoiding overly conservative behavior in well-modeled regions. Further details regarding this design choice are provided in Appendix~\ref{appendix:experiments} of the supplementary material.


\subsection{FrozenLake}
\label{sec:frozenlake}

\paragraph{Environment.} The \texttt{FrozenLake} task is played on an $8\times8$ grid. The agent begins in the upper-left cell and must navigate to the goal in the lower-right cell without falling into any of the "hole" cells scattered throughout the grid. At each time step, the agent chooses one of four actions: \texttt{up}, \texttt{down}, \texttt{left} or \texttt{right}, but movement is stochastic: with probability $p$ the agent moves in the intended direction, and with probability $(1 - p)/2$ it instead slips to one of the two orthogonal neighbors.

The immediate reward at each state is defined as: \( r(s) = \frac{1}{(d(s) + 1)^3} \), where \( d(s) \) is the Manhattan distance from state \( s \) to the goal. A terminal reward of 1 is granted upon reaching the goal, while falling into a hole yields a reward of 0. Each episode ends when a terminal state is reached or after 150 time steps.

\paragraph{Model Uncertainty.}
In our setup, the true transition dynamics are defined using \( p = 0.4 \). However, the agent plans using an approximate model that differs only in states adjacent to holes. In these uncertain regions, the probability of moving in the intended direction is increased to \( p^o = p + \rho \), while the probabilities of deviating to either perpendicular direction are adjusted to \( (1 - p^o)/2 \). This modification satisfies the uncertainty condition defined in \eqref{eq:uncertainty_assumption}. Elsewhere, the approximate model matches the true dynamics exactly. A visualization of the environment is shown in Figure~\ref{fig:frozenlake}.

\begin{figure}[ht]
    \centering
    \includegraphics[width=0.6\linewidth]{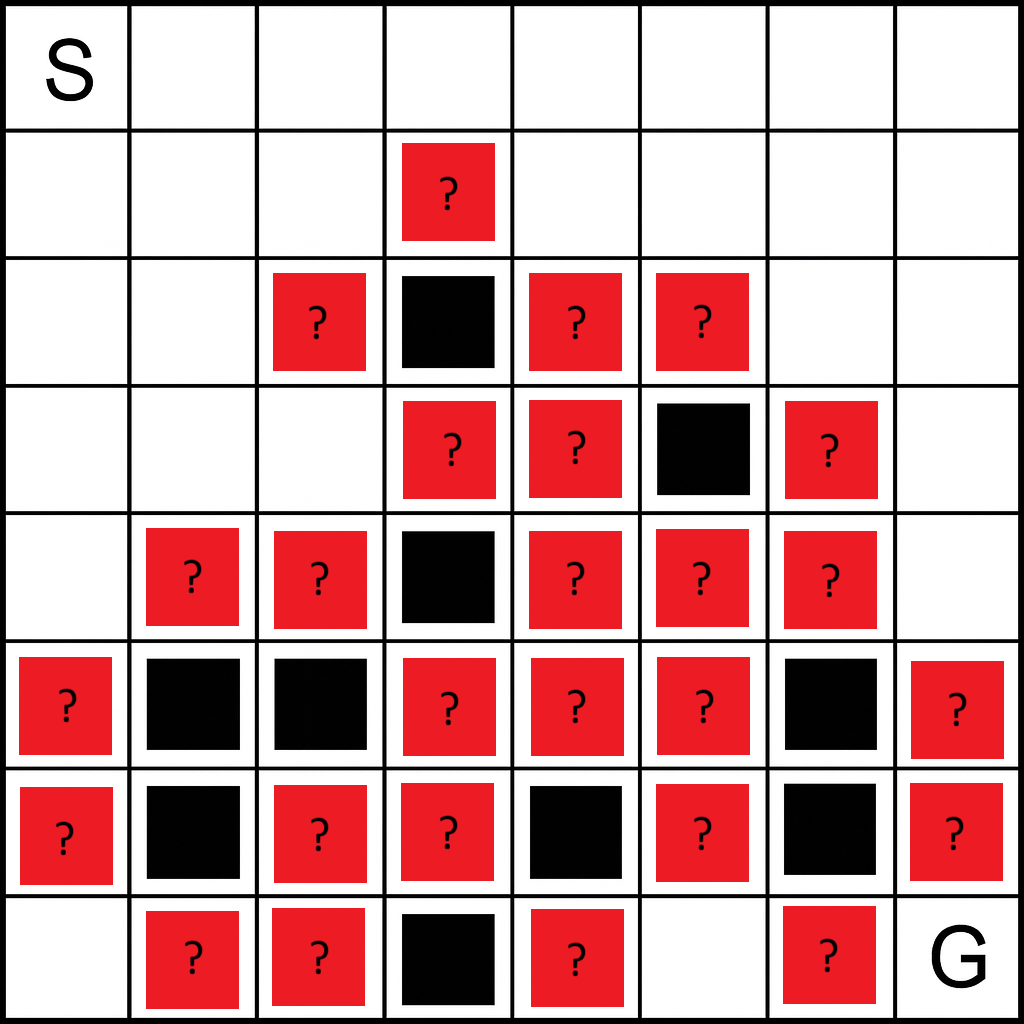}
    \caption{Visualization of the $8 \times 8$ \texttt{FrozenLake} environment, a stochastic grid-world where the agent starts in the top-left cell (S) and aims to reach the goal in the bottom-right cell (G), while avoiding hazardous holes represented by black squares. Due to the stochastic nature of the environment, the agent's actions may not always result in the intended direction. Cells adjacent to holes are marked with red squares containing question marks, highlighting regions of model uncertainty where the agent's planning model deviates from the true environment dynamics.}
    \label{fig:frozenlake}
\end{figure}

\paragraph{Experimental Setup and Results.}
Experimental results are summarized in Table~\ref{tab:frozenlake_results}. We evaluate RSS and standard SS under varying levels of model uncertainty, each over 1000 different seeds. As a benchmark, we also evaluate SS with full access to the true dynamics, achieving an average discounted return of \( 0.249 \pm 0.012 \). All methods use a planning horizon of \( H = 3 \), sample width \( C = 50 \), and discount factor \( \gamma = 0.99 \). The uncertainty budget \( \rho \) is varied across the set \( \{0.1, 0.2, 0.3, 0.4, 0.5, 0.6\} \).

As expected, both RSS and SS underperform compared to the SS variant with full access to the true environment dynamics. However, RSS consistently demonstrates better performance than SS across all values of $\rho$, with the performance gap widening as uncertainty increases. This highlights RSS's robustness to model misspecification and its ability to maintain stronger performance under growing uncertainty.
\begin{table}[h!]
\centering
\begin{tabular}{c|c|c}
$\rho$ & RSS & SS \\
\hline
0.1 & \textbf{0.177} $\pm$ 0.011 & 0.172 $\pm$ 0.011 \\
0.2 & \textbf{0.171} $\pm$ 0.011 & 0.123 $\pm$ 0.009 \\
0.3 & \textbf{0.145} $\pm$ 0.010 & 0.109 $\pm$ 0.009 \\
0.4 & \textbf{0.126} $\pm$ 0.009 & 0.098 $\pm$ 0.008 \\
0.5 & \textbf{0.127} $\pm$ 0.009 & 0.080 $\pm$ 0.007 \\
0.6 & \textbf{0.118} $\pm$ 0.009 & 0.080 $\pm$ 0.008 \\
\end{tabular}
\caption{
Performance of RSS and SS in the FrozenLake environment under varying uncertainty levels $\rho$. The reported values are the average discounted return with the standard error over 1000 different seed. The best-performing algorithm for each $\rho$ is highlighted in bold. The average discounted return of SS with access to the true dynamics is $0.249 \pm 0.012$.
}
\label{tab:frozenlake_results}
\end{table}

\subsection{CartPole}
\label{sec:cartpole}

\paragraph{Environment.}
We use the \texttt{CartPole} environment, where the agent must balance a pole on a moving cart by applying discrete left or right forces. The continuous state is defined by the cart’s position \(x\), velocity \(\dot{x}\), pole angle \(\theta\), and angular velocity \(\dot{\theta}\). An episode terminates if \( |\theta| > 0.2 \) radians, \( |x| > 2.4 \), or after 200 time steps. The reward is defined as \( r(\theta) = 1 - 0.2|\theta| \) in non-terminal states, and 0 otherwise, encouraging the pole to remain upright.

At the start of each episode, the cart is centered and the pole is vertical. At each step, the agent selects a force, transitioning to the next state according to deterministic dynamics, with added Gaussian noise \( \mathcal{N}(0, \sigma^2_\theta(x)) \) on the pole angle. The noise variance depends on the current cart position \(x\), defined as:
\begin{equation}
\sigma_{\theta}^2(x) = \begin{cases}
    \sigma_{high}^2, & \text{if } x_a < |x| < x_b \\
    \sigma_{low}^2, & \text{otherwise}
\end{cases}
\end{equation}
This models a narrow “hazard zone” \( x \in \pm[x_a, x_b] \), where the system is more unstable due to higher noise.

\paragraph{Model Uncertainty.}
The hazard zone is assumed to be narrow and difficult to model accurately. As such, the planning model assumes a constant low noise \( \sigma_{low}^2 \) across all states, underestimating the true noise in the hazard zone. This mismatch induces localized model uncertainty. Full noise specifications and uncertain total variance calculations are detailed in \ref{appendix:cartpole}.

\paragraph{Experimental Setup and Results.}
We set \( x_a = 0.02 \), \( x_b = 0.03 \), \( \sigma_{low} = 10^{-3} \), and vary \( \sigma_{high} \) from 0.07 to 0.15. We compare RSS against standard SS, using a planning horizon \( H = 5 \), width \( C = 10 \), and discount factor \( \gamma = 0.999 \). As a reference, we also evaluate SS with access to the true noise model. Each configuration is averaged over 500 random seeds.

\begin{figure}[ht]
    \centering
    \includegraphics[width=0.8\linewidth]{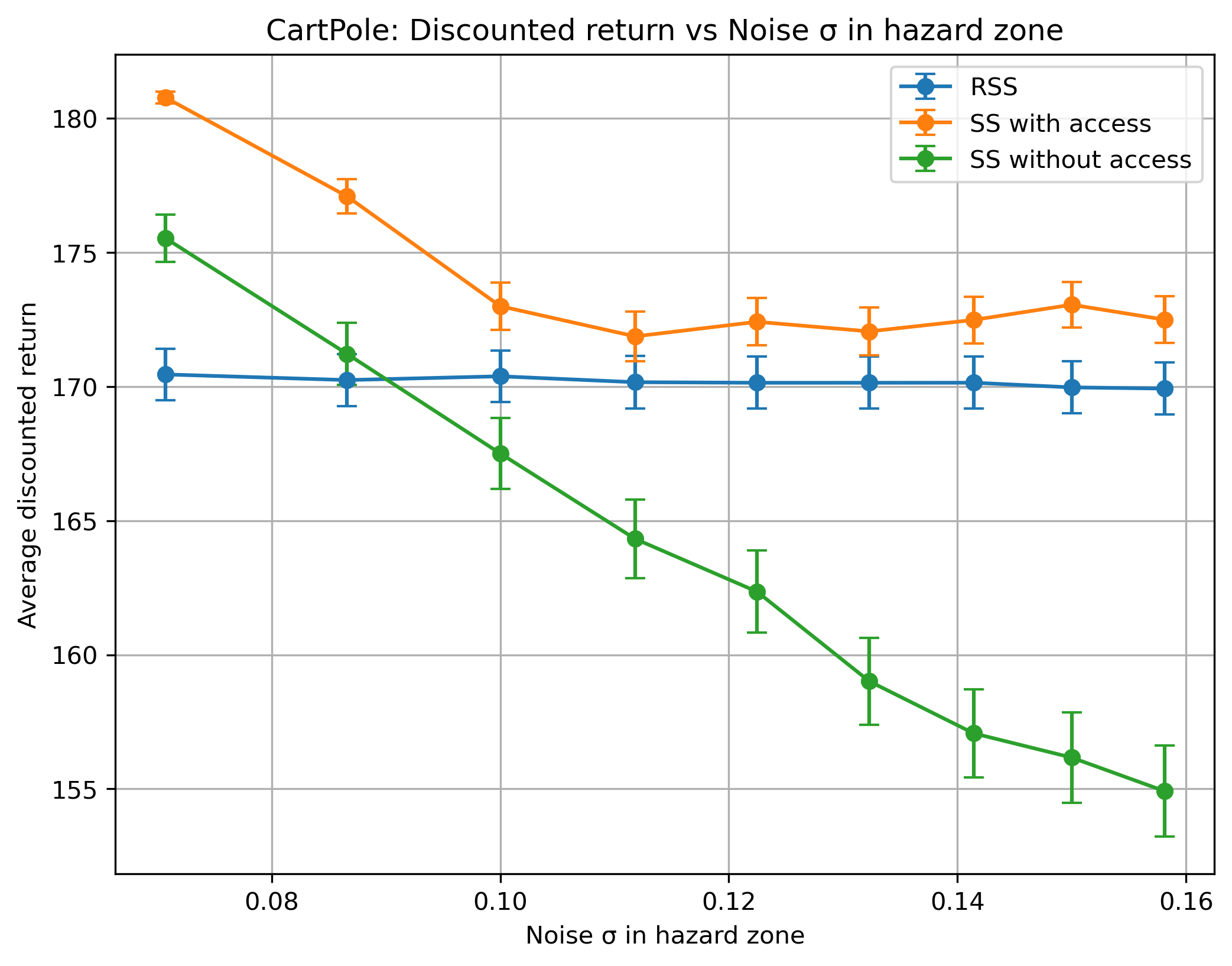}
    \caption{Average discounted return comparison of RSS and SS (with/without access to true dynamics) under increasing noise variance in the hazard zone. Error bars denote standard error across 500 different seeds.}
    \label{fig:cartpole_avg_return_vs_noise}
\end{figure}

Figures~\ref{fig:cartpole_avg_return_vs_noise} show the average discounted performance as a function of the noise standard deviation $\sigma_{high}$ within the hazard zone. An additional figure in \ref{fig:cartpole_avg_success_vs_noise} presents the success rate—defined as completing 200 steps without termination—under varying noise levels. As expected, the SS variant with full access to the true dynamics achieves the highest performance across all noise levels, as it can plan optimally using accurate environment information. The performance of both SS variants (with and without model access) degrades as noise increases, indicating their increased sensitivity to unmodeled uncertainty.

In contrast, RSS maintains stable performance across all levels of noise, demonstrating its robustness to model misspecification. Notably, in low-noise regimes, RSS underperforms compared to SS without access. This is a known phenomenon in robust planning: robust policies are inherently conservative, as they optimize for the worst-case plausible dynamics within an uncertainty set, leading to overly cautious behavior that sacrifices performance for safety \cite{mannor2012lightning}.

However, as the noise variance increases, RSS maintains a near-constant performance level in this scenario. SS without access to the true model continues to rely on an underestimated noise model, resulting in unsafe and suboptimal actions, while RSS anticipates and mitigates adverse dynamics. As a result, RSS eventually outperforms SS without access in both return and success rate. This crossover point highlights the fundamental trade-off in robust planning: while robust methods may underperform in low-risk settings, they provide significant benefits in high-uncertainty environments by reducing risk and failure rates.

\section{Conclusions}
\label{sec:conclusions}
In this work, we introduced the Robust Sparse Sampling (RSS) algorithm, the first online planning algorithm for RMDPs with finite-sample theoretical performance guarantees. RSS extends the Sparse Sampling algorithm by incorporating robustness against model errors, leveraging the dual formulation of robust value functions and Sample Average Approximation (SAA) techniques. Our theoretical analysis establishes finite-time error bounds for RSS, and we demonstrate its effectiveness in simulative experiments in environments with uncertain transition dynamics.

We hope that the RSS algorithm and methods will serve as a foundation for future research in robust online planning, both for methods that can scale better for large state and action spaces, and for online methods that can handle model uncertainty. Moreover, we wish to extend our methods to anytime-fashion Monte Carlo Tree Search (MCTS), and to Partially Observable Markov Decision Process (POMDP) settings.

\paragraph{Limitations}
\label{sec:limitations}

While the RSS algorithm is the first to address robust online planning under model uncertainty with formal performance guarantees, it exhibits several important limitations.

Similar to Sparse Sampling, RSS suffers from significant sample and computational inefficiency, as the complexity grows exponentially with the planning horizon $H$. This severely restricts its practical applicability in environments requiring long-horizon planning. However, in settings where short planning horizons are sufficient—e.g., due to low discount factors or inherently short episodes—RSS may still offer a viable and effective alternative.

Second, the algorithm assumes prior knowledge of the uncertainty budget parameter $\rho$. In real-world applications, accurately estimating $\rho$ is often nontrivial, especially in non-stationary or partially observed environments where the transition dynamics may evolve over time. Estimating such parameters reliably remains an open problem in robust decision-making \cite{kumar2024learning, suilen2022robust}.

A current limitation that leads to over-conservatism is the rectangularity assumption of the uncertainty set. Recent works have shown promising directions to address this issue \cite{goyal2023robust}, and we hope to incorporate those in online robust planning as well.

\bibliography{refs}

\ifthenelse{\boolean{includesupplement}}{   
\appendix
\section{Proofs}
\label{appendix:proofs}

We begin by introducing several auxiliary lemmas that form the basis for Theorem~\ref{thm:robust-sparse-sampling-guarantee}, our main theoretical result. This theorem bounds the discrepancy between the value function induced by the policy computed by RSS and the true robust value function $V^*$. Our analysis builds on the classical SS proof~\cite{kearns2002sparse}, extending it to handle model uncertainty using the convergence theory of Sample Average Approximation. 

We first analyze functions $F_{s,a}^{\rho}(\eta)$ and $\hat{F}_{s,a}^{\rho}(\eta)$ defined in Equations~\eqref{eq:dual_reformulation_short} and~\eqref{eq:saa_optimization}, respectively.

\begin{proposition}
\label{prop:convexity_lipschitz}
For any state $s \in \mathcal{S}$ and action $a \in \mathcal{A}$, the functions $F_{s,a}^{\rho}(\eta)$ and $\hat{F}_{s,a}^{\rho}(\eta)$ are convex and $(2 - \rho)$-Lipschitz continuous with respect to $\eta$.
\end{proposition}
\begin{proof}
We first show that the function is convex w.r.t $\eta$. For $\eta_1, \eta_2 \in \left[0, \frac{2}{\rho(1 - \gamma)}\right]$, we have:
\begin{align}
    &\big| F_{s,a}^{\rho}(\eta_1) - F_{s,a}^{\rho}(\eta_2) \big| 
    \leq (1 - \rho) \cdot |\eta_1 - \eta_2| + \notag \\
    &\left| \mathbb{E}_{s' \sim P^o_{s,a}}\left[(\eta_1 - V^*(s'))_+ - (\eta_2 - V^*(s'))_+\right] \right| \label{eq:line1} \\
    &\leq (1 - \rho) \cdot |\eta_1 - \eta_2| + \\
    &\mathbb{E}_{s' \sim P^o_{s,a}}\left[\left| (\eta_1 - V^*(s'))_+ - (\eta_2 - V^*(s'))_+ \right| \right] \label{eq:line2} \\
    &\leq (2 - \rho) \cdot |\eta_1 - \eta_2|. \notag
\end{align}
Here, Inequality~\eqref{eq:line1} follows from Jensen's inequality, and~\eqref{eq:line2} follows from the inequality $(a)_+ - (b)_+ \leq (a - b)_+$.
Therefore, $F_{s,a}^{\rho}(\eta)$ is $(2 - \rho)$-Lipschitz continuous.

We now prove convexity. For any $t \in [0, 1]$, we have:
\begin{align}
    &F_{s,a}^{\rho}(t\eta_1 + (1-t)\eta_2) = \notag \\
    &\mathbb{E}_{s' \sim P^o_{s,a}}\left[(t\eta_1 + (1-t)\eta_2 - V^*(s'))_+\right] \\ &\quad - (t\eta_1 + (1-t)\eta_2)(1 - \rho) \label{eq:line3} \\
    &\leq t F_{s,a}^{\rho}(\eta_1) + (1 - t) F_{s,a}^{\rho}(\eta_2), \notag
\end{align}
where Inequality~\eqref{eq:line3} follows from the convexity of the function $(x)_+$ and the linearity of expectation.

Therefore, $F_{s,a}^{\rho}(\eta)$ is convex. The same argument applies to the empirical estimator 
$\hat{F}_{s,a}^{\rho}(\eta) = \frac{1}{C} \sum_{i=1}^C (\eta - V^*(s_i'))_+ - \eta(1 - \rho)$,
which is a finite sum of convex and $(2 - \rho)$-Lipschitz functions, and thus retains these properties.
\end{proof}
Using Proposition~\ref{prop:convexity_lipschitz}, we derive concentration bounds that quantify how closely the SAA-based estimate \( \hat{Q}^*(s,a) \), defined in~\eqref{eq:saa_optimization}, approximates the true robust action-value \( Q^*(s,a) \) from~\eqref{eq:dual_reformulation_short}. The proof leverages Hoeffding's inequality and the Lipschitz continuity properties from Proposition~\ref{prop:convexity_lipschitz}.

\begin{lemma}
\label{lem:theoretical_bound_saa}
Let \( \hat{Q}^*(s,a) \) denote the SAA estimate of the optimal robust action-value function as defined in Equation~\eqref{eq:saa_optimization}, and let \( Q^*(s,a) \) denote the true optimal action-value function given in Equation~\eqref{eq:dual_reformulation_short}. Then, for any \( \lambda > 0 \), 
\begin{align*}
&\mathbb{P}\left(\left| \hat{Q}^*(s,a) - Q^*(s,a) \right| \geq \lambda \right) \leq \\
&\frac{2}{\lambda} \left( \frac{8 - 4\rho}{(1 - \gamma)\rho} \right) \cdot \exp\left(-\frac{C \lambda^2 \rho^2 (1 - \gamma)^2}{2} \right).
\end{align*}
\end{lemma}
\begin{proof}
Let \( \eta^{\star}_{s,a} \) and \( \hat{\eta}^{\star}_{s,a} \) denote the optimal solutions to the stochastic optimization problems defined by \( F^{\rho}_{s,a}(\eta) \) and \( \hat{F}^{\rho}_{s,a}(\eta) \) in Equations~\eqref{eq:dual_reformulation_short} and~\eqref{eq:saa_optimization}, respectively. With this notation, we can express the true and SAA estimation of the robust action-value functions as:
\begin{align}
    Q^*(s,a) &= r(s,a) - \gamma F^{\rho}_{s,a}(\eta^{\star}_{s,a}), \\
    \hat{Q}^*(s,a) &= r(s,a) - \gamma \hat{F}^{\rho}_{s,a}(\hat{\eta}^{\star}_{s,a}).
\end{align}
Using the fact that $\hat{F}_{s,a}^{\rho}(\hat{\eta}^{\star}_{s,a}) \leq \hat{F}_{s,a}^{\rho}(\eta^{\star}_{s,a})$, we have:
\begin{align*}
&\mathbb{P}\left(\hat{Q}^*(s,a) - Q^*(s,a) \geq \lambda\right) = \\
&\mathbb{P}\left(\hat{F}_{s,a}^{\rho}(\hat{\eta}^{\star}_{s,a}) - F_{s,a}^{\rho}(\eta^{\star}_{s,a}) \geq \lambda\right) \leq\\
&\mathbb{P}\left(\hat{F}_{s,a}^{\rho}(\eta^{\star}_{s,a}) - F_{s,a}^{\rho}(\eta^{\star}_{s,a}) \geq \lambda\right) =\\
&\mathbb{P}\left( \frac{1}{C}\sum_{i=1}^C (\eta^* - V^*(s_i'))_+ - \mathbb{E}_{s'}[(\eta^* - V^*(s'))_+] \geq \lambda \right).
\end{align*}
Applying Hoeffding’s inequality (bounded by $\frac{2}{\rho(1 - \gamma)}$):
\begin{align*}
\mathbb{P}(\hat{Q}^*(s,a) - Q^*(s,a) \geq \lambda) \leq 
\exp\left(-\frac{C \lambda^2 (\rho(1 - \gamma))^2}{2}\right).
\end{align*}
For the lower tail, define a grid $\{\eta_i\}_{i=1}^d$ over $[0, \frac{2}{\rho(1 - \gamma)}]$ with spacing $\frac{1}{d}$. Let:
\begin{align*}
A_N = \left\{ \max_{1 \leq i \leq d} \left( \hat{F}_{s,a}^{\rho}(\eta_i) - F_{s,a}^{\rho}(\eta_i) \right) < \frac{\lambda}{2} \right\}.
\end{align*}
Then:
\begin{align}
&\mathbb{P}(A_N) \geq 1 - \mathbb{P}(\hat{F}_{s,a}^{\rho}(\eta_i) - F_{s,a}^{\rho}(\eta_i) ) \geq \\
&1 - d \cdot \exp\left(-\frac{C \lambda^2 (\rho(1 - \gamma))^2}{8}\right),
\end{align}
where the last steps follows again by Hoeffding’s inequality (bounded by $\frac{2}{\rho(1 - \gamma)}$).
Choose $d = \left\lceil \frac{1}{\lambda} \left(\frac{8-4\rho}{(1 - \gamma)\rho}\right) \right\rceil$. Using Lipschitz continuity:
\begin{align}
|\hat{F}_{s,a}^{\rho}(\eta_i) - \hat{F}_{s,a}^{\rho}(\hat{\eta}^{\star}_{s,a})| \leq (2 - \rho) \cdot \frac{2}{\rho(1 - \gamma)} \cdot \frac{1}{d} \leq \frac{\lambda}{2}.
\end{align}
This yields:
\begin{align}
&\mathbb{P}( Q^*(s,a) - \hat{Q}^*(s,a) < \lambda )= \\
&\mathbb{P}( F_{s,a}^{\rho}(\eta_{s,a}^{\star}) - \hat{F}_{s,a}^{\rho}(\hat{\eta}_{s,a}) < \lambda )\geq \\ 
&1 - \frac{1}{\lambda} \left(\frac{8 - 4\rho}{(1 - \gamma)\rho} \right) \cdot \exp\left(-\frac{C \lambda^2 (\rho(1 - \gamma))^2}{8} \right).
\end{align}
Combining both directions:
\begin{align}
&\mathbb{P}\left(\left| \hat{Q}^*(s,a) - Q^*(s,a) \right| \geq \lambda\right) 
\leq \\
&\frac{2}{\lambda} \left(\frac{8 - 4\rho}{(1 - \gamma)\rho}\right) \cdot \exp\left(-\frac{C \lambda^2 (\rho(1 - \gamma))^2}{2}\right).
\end{align}
\end{proof}
It is important to note that the SAA-estimated robust action-value \( \hat{Q}(s,a) \) is a theoretical construct, as the true robust value function \( V^*(s') \) is not available in practice. In the RSS algorithm, at depth \( d \), \( V^*(s') \) is approximated by \( \hat{V}_{d-1}(s') \), the estimated robust value of the sampled successor state from the previous depth. Although \( \hat{Q}(s,a) \) is not computed directly during planning, it serves a key theoretical role in quantifying the error between the true robust action-value \( Q^*(s,a) \) and the depth-\( d \) estimate \( \hat{Q}_d(s,a) \) produced by RSS. 

To capture how the estimation error accumulates across depths, we define a depth-dependent error term $\alpha_d$ recursively for given $\lambda>0$:
\begin{equation}
\alpha_0 = 0, \quad \alpha_d = \gamma(\lambda + \alpha_{d-1}) \quad \text{for } H \geq d \geq 1.
\end{equation}
This term quantifies the cumulative estimation error up to depth $d$. The error at depth $d$ arises from two sources: the finite-sample approximation error $\lambda$ and the accumulated error from the previous depth $\alpha_{d-1}$. The combined error is $\lambda + \alpha_{d-1}$, discounted by $\gamma$, resulting in $\alpha_d$. The following lemma shows that $\alpha_d$ serves as a probabilistic upper bound on the difference between $Q^*(s,a)$ and $\hat{Q}_d(s,a)$.
\begin{lemma}
\label{lem:theoretical_bound}
Let \( Q^*(s,a) \) denote the optimal action-value function defined in Equation~\eqref{eq:dual_reformulation_short}, and let \( \hat{Q}_d(s,a) \) denote the estimate produced by the RSS algorithm at depth \( d \), as defined in Equation~\eqref{eq:robust_sparse_sampling}. Then:
\begin{align*}
& \mathbb{P}\left(\left| \hat{Q}_d(s,a) - Q^*(s,a) \right| \leq \alpha_d \right) \geq \\
& 1-(\lvert A \rvert \cdot C)^d \frac{2}{\lambda} \left( \frac{8 - 4\rho}{(1 - \gamma)\rho} \right) \cdot \exp\left(-\frac{C \lambda^2 \rho^2 (1 - \gamma)^2}{2} \right), \nonumber
\end{align*}
where $\alpha_0 = 0$ and $\alpha_d = \gamma(\lambda + \alpha_{d-1})$ for $H\geq d > 0$.
\end{lemma}
\begin{proof}
Let:
\begin{align*}
\eta^{\star}_{s,a} &= \arg\min_{\eta \in [0, \frac{2}{\rho(1 - \gamma)}]} F_{s,a}^{\rho}(\eta), \\
\hat{\eta}^{\star}_{s,a} &= \arg\min_{\eta \in [0, \frac{2}{\rho(1 - \gamma)}]} \hat{F}_{s,a}^{\rho}(\eta), \\
\tilde{\eta}^{\star, d}_{s,a} &= \arg\min_{\eta \in [0, \frac{2}{\rho(1 - \gamma)}]} \tilde{F}_{s,a}^{\rho, d}(\eta).
\end{align*}
We proceed by induction on the tree depth $d$.

For the base case $d=0$, the bound holds trivially since $\alpha_0 = 0$ and $Q^*$ is bounded by $\frac{1}{(1-\gamma)}$.

Assume the bound holds for depth $d-1$. For depth $d$:
\begin{align*}
&\big| Q^*(s,a) - \hat{Q}_d(s,a) \big| 
\leq\\
&\big| Q^*(s,a) - \hat{Q}^*(s,a) \big| + \big| \hat{Q}(s,a) - \hat{Q}_d(s,a) \big|.
\end{align*}
We now bound each term separately.

For the first term:
\begin{align}
&Q^*(s,a) - \hat{Q}^*(s,a) =\\
&F_{s,a}^{\rho}(\eta^{\star}_{s,a}) - \hat{F}_{s,a}^{\rho}(\hat{\eta}^{\star}_{s,a}) \leq\\
&F_{s,a}^{\rho}(\hat{\eta}^{\star}_{s,a}) - \hat{F}_{s,a}^{\rho}(\hat{\eta}^{\star}_{s,a}) \leq\\
&\gamma \left( \mathbb{E}_{s'}[(\hat{\eta}^* - V^*(s'))_+] - \frac{1}{C} \sum_{i=1}^C (\hat{\eta}^* - V^*(s'_i))_+ \right) \leq\\
&\gamma \lambda.
\end{align}
A symmetric argument yields:
\begin{align}
\left| Q^*(s,a) - \hat{Q}(s,a) \right| \leq \gamma \lambda.
\end{align}
For the second term:
\begin{align}
&\hat{Q}^*(s,a) - \hat{Q}_d(s,a) =\\
&\hat{F}^{\rho}_{s,a}(\hat{\eta}^{\star}_{s,a}) - \tilde{F}^{\rho, d}_{s,a}(\tilde{\eta}^{\star, d}_{s,a}) \leq\\
&\hat{F}^{\rho}_{s,a}(\tilde{\eta}^{\star, d}_{s,a}) - \tilde{F}^{\rho, d}_{s,a}(\tilde {\eta}^{\star, d}_{s,a}) =\\
&\gamma \left( \frac{1}{C} \sum_{i=1}^C (\tilde{\eta}^{\star, d}_{s,a} - V(s'_i))_+ - (\tilde{\eta}^{\star, d}_{s,a} - \hat{V}_{d-1}^*(s'_i))_+ \right) \leq \\
&\gamma \left( \frac{1}{C} \sum_{i=1}^C (\hat{V}_{d-1}^*(s'_i) - V(s'_i))_+ \right) \leq \\
&\gamma \alpha_d.
\end{align}
Therefore,
\begin{align}
\left| Q^*(s,a) - \hat{Q}_d(s,a) \right| \leq \gamma \lambda + \gamma \alpha_d = \alpha_{d+1}.
\end{align}

The probability of a bad estimate compounds across $K$ actions and $C$ samples at each node, yielding a multiplicative factor $(KC)^d$.
Applying the concentration bound from Lemma~\ref{lem:theoretical_bound_saa} completes the proof.
\end{proof}
We now can recursively compute the bound $\alpha_H$, which captures the total error at the root (depth $H$) between the optimal robust action-value function $Q^*(s,a)$ and the RSS-estimated value $\hat{Q}_H(s,a)$.
\begin{align}
\alpha_H = \left( \sum_{i=1}^H \gamma^i \lambda \right) + \gamma^H \frac{1}{1-\gamma} \leq \frac{1}{1 - \gamma} (\lambda+ \gamma^H).
\end{align}
By choosing the horizon \( H = \left\lceil \log_{\gamma}(\lambda) \right\rceil \), we ensure with high probability that the RSS estimate at depth \(H\) satisfies \( \alpha_H \leq \frac{2\lambda}{1 - \gamma} \). Next, given any desired confidence level \( 1 - \delta \), where \(0 < \delta < 1\) represents the maximum acceptable failure probability, we select a sufficiently large constant \( C \) to achieve this confidence. Specifically, we select \( C \) satisfying:
\begin{align}
C \geq &\frac{2}{\lambda^2 \rho^2 (1-\gamma)^2} \cdot \nonumber \\
&\left( 2H \ln \left(\frac{2\lvert A \rvert \cdot H}{\lambda^2 \rho^2 (1-\gamma)^2} \right) + \ln \left( \frac{2(8-4\rho)}{\delta \lambda (1-\gamma) \rho}\right)\right).
\end{align}
Using the parameters set above, we guarantee that, with probability at least \( 1 - \delta \), the RSS estimate at depth \(H\) is within \( \frac{2\lambda}{1-\gamma} \) of the true robust action-value function. To complete our analysis, we present the following lemma, which relates the robust action-value function to the robust value function. This result holds for any stochastic policy and extends the original lemma from~\cite{kearns2002sparse}.
\begin{lemma}
\label{lem:policy_bound}
Denote $\pi^*$ as the optimal robust policy. Let $\pi$ be a stochastic policy such that $\mathbb{P}(Q^*(s,\pi^*(s)) - Q^*(s,\pi(s)) < \lambda) \geq 1 - \delta$ for all $s$. Then:
\begin{align}
V^*(s) - V^\pi(s) \leq \frac{\delta}{(1 - \gamma)^2} + \frac{\lambda}{1 - \gamma}.
\end{align}
\end{lemma}

\begin{proof}
Since the reward is bounded by $0\leq r\leq1$, we have that:
\begin{align*}
    &\mathbb{E}_{a\sim \pi(s)}[Q^*(s,a)] \geq\\
    &(1-\delta)(Q^*(s,\pi^*(s))-\lambda) \geq \\
    &Q^*(s,\pi^*(s)) - \delta \frac{1}{1-\gamma} - \lambda. 
\end{align*}
Now denote $\beta=\frac{\delta}{1-\gamma}+\lambda$, thus we have that $Q^*(s,\pi^*(s))-Q^*(s,\pi(s)) \leq \beta$ which implies $\left| \mathbb{E}[r(s,\pi^*(s))]-\mathbb{E}[r(s,\pi(s))]\right| \leq \beta$. Now consider a policy $\pi^j$ that executes $\pi$ for the first $j+1$ steps and then executes $\pi^*(s)$ for the rest of the time. Thus, we have  
\begin{align*}
    &V^{*}(s)-V^{\pi^j}(s) = \\
    &\min_{P' \in \mathcal{P}} V^{\pi^*,P'}(s)-  \min_{P' \in \mathcal{P}}  V^{\pi^j,P'}(s) = \\
    & V^{\pi^*,P^*}(s) -  V^{\pi^j,P^{\pi^j}}(s) \leq \\
    & V^{\pi^*,P^{\pi^j}}(s) -  V^{\pi^j,P^{\pi^j}}(s) \leq \\
    & \sum_{i=0}^{j}\beta\gamma^i. 
\end{align*}
Thus, we have that: $V^{*}(s)-V^{\pi}(s) \leq \frac{\beta}{1-\gamma}=\frac{\delta}{(1-\gamma)^2} + \frac{\lambda}{1-\gamma}$.
\end{proof}

Leveraging Lemma~\ref{lem:theoretical_bound}, we see that the RSS stochastic policy \(\pi\) meets the condition required by Lemma~\ref{lem:policy_bound}. Applying this lemma directly to the RSS policy, we obtain a bound relating the RSS policy’s value function to the true robust optimal value function \(V^*(s)\). This argument culminates in our main theoretical result in theorem \ref{thm:robust-sparse-sampling-guarantee}.

\begin{restatedtheorem}[Theorem~\ref{thm:robust-sparse-sampling-guarantee}]
For any $s \in \mathcal{S}$ and any $\epsilon > 0$, the Robust Sparse Sampling algorithm returns a policy $\pi$ such that:
\begin{align*}
\left|V^{\pi}(s) - V^*(s)\right| \leq \epsilon,
\end{align*}
with the following hyperparameters:
\begin{align*}
&\lambda = \frac{\epsilon}{3}, \ \delta = \lambda (1 - \gamma), \\
&H = \left\lceil \log_{\gamma}(\lambda) \right\rceil, \\
&C = \frac{2}{\lambda^2 \rho^2 (1-\gamma)^2} \cdot\\
&\left( 2H \ln \left(\frac{2KH}{\lambda^2 \rho^2 (1-\gamma)^2} \right) + \ln \left( \frac{2(8-4\rho)}{\delta \lambda (1-\gamma) \rho}\right)\right).
\end{align*}
\end{restatedtheorem}

\begin{proof}
    By Lemma \ref{lem:theoretical_bound} we have that the error in the estimation of $Q^*$ is at most $\alpha_H$, with probability $1-(K C)^d \frac{2}{\lambda} 
    \left( \frac{8 - 4\rho}{(1 - \gamma)\rho} \right) 
     \cdot \exp\left(-\frac{C \lambda^2 (\rho(1 - \gamma))^2}{2} \right)$. Using the values that appears in the therom for $C$ and $H$ we have that with probability $1-\delta$ the error is at most $\frac{2\lambda}{1-\gamma}$. Thus, we can apply Lemma \ref{lem:policy_bound} to conclude that the policy $\pi$ computed by the Robust Sparse Sampling algorithm is such that for every state $s$ we have that:
    \begin{align*}
        V^*(s)-V^\pi(s) \leq \frac{2\lambda}{(1-\gamma)}+ \frac{\delta}{(1 - \gamma)^2}.
    \end{align*}
Subtituting $\delta = \lambda (1-\gamma)$ and $\lambda=\frac{\epsilon}{3}$ we have that:
\begin{align*}
    V^*(s)-V^\pi(s) \leq \epsilon.
\end{align*}
\end{proof}

\section{Further Implementation Details}
\label{appendix:experiments}

\subsection{RSS Implementation with State-Action Dependent $\rho$}
In this section, we describe the implementation details of the RSS algorithm when the robustness parameter \( \rho \) is not constant, but varies with the state and action, i.e., \( \rho = \rho(s, a) \in [0, 1] \). This setting is more realistic in many practical applications, where certain regions of the environment are well understood due to ample data, while other regions remain uncertain because they are rarely visited.

Using the robust backup update indiscriminately across all states and actions can result in overly conservative behavior, particularly in areas where the model is already accurate. To avoid this, we introduce a modification to the RSS algorithm, provided in Algorithm~\ref{alg:robust_sparse_sampling_chaning_rho}. The algorithm behaves similarly to the original RSS, but applies the robust backup update only in state-action pairs where \( \rho(s, a) > 0 \). In all other cases, it defaults to the standard expected backup. This ensures the algorithm remains cautious in uncertain regions while avoiding unnecessary conservatism in well-understood parts of the environment.

From a theoretical standpoint, the guarantees in Theorem~\ref{thm:robust-sparse-sampling-guarantee} still hold with minor adjustments. Let \( \rho_{\max} = \max_{s,a} \rho(s,a) \). By replacing the constant \( \rho \) in the original proofs with \( \rho_{\max} \), we can extend the results of Lemmas~\ref{lem:theoretical_bound_saa},~\ref{lem:theoretical_bound}, and~\ref{lem:policy_bound} to this variable-\( \rho \) setting. Consequently, the bound on the value function remains valid, though it will be scaled by \( \rho_{\max} \) instead of the specific values \( \rho(s,a) \). While this substitution may introduce looseness in regions where \( \rho(s,a) \ll \rho_{\max} \), the overall theoretical structure remains intact.

\begin{algorithm}[ht]
\caption{Robust Sparse Sampling (RSS) - Changing $\rho$}
\label{alg:robust_sparse_sampling_chaning_rho}
\textbf{Input}: Current state $s$, current depth $d$ \\
\textbf{Parameter}: Sample width $C$, planning horizon $H$ computed based on Theorem~\ref{thm:robust-sparse-sampling-guarantee} \\
\textbf{Output}: \parbox[t]{.85\linewidth}{Estimated optimal action and its value}
\begin{algorithmic}[1]
\IF{$d = 0$}
    \STATE \textbf{return} 0
\ENDIF
\FORALL{$a \in \mathcal{A}$}
    \STATE $V_{\text{list}} \gets [\,]$
    \FOR{$i = 1$ to $C$}
        \STATE Sample $s'_i \sim P^o(\cdot \mid s, a)$
        \STATE $(\_, \ \hat{V}_{d-1}(s'_i)) \gets$ \textsc{RSS}$(s'_i, d-1)$
        \STATE Append $\hat{V}_{d-1}(s'_i)$ to $V_{\text{list}}$
    \ENDFOR

\ENDFOR
\IF{$\rho(s,a) > 0$}

    \STATE Update $\hat{Q}_d (s,a)$ using Equation~\eqref{eq:robust_sparse_sampling} with $V_{\text{list}}$ and $\rho(s,a)$
\ELSE
    \STATE Update $\hat{Q}_d (s,a)$ using Equation~\eqref{eq:sparse_sampling_recursion} with $V_{\text{list}}$
\ENDIF

\STATE \textbf{return} $\arg\max_{a \in \mathcal{A}} \hat{Q}_d (s,a), \max_{a \in \mathcal{A}} \hat{Q}_d (s,a)$
\end{algorithmic}
\end{algorithm}

\subsection{CartPole Environment}
\label{appendix:cartpole}
In this section, we provide a detailed description of the CartPole environment used in our experiments, focusing on how model uncertainty is calculated.

The CartPole standard dynamics are governed by deterministic second-order differential equations. Let the system's state be defined by the cart's position \( x \), velocity \( \dot{x} \), the pole's angle \( \theta \), and angular velocity \( \dot{\theta} \), with a control action \( a \). The next state is determined by a transition function \( f(x, \dot{x}, \theta, \dot{\theta}, a) \), such that:
\begin{equation}
    \begin{pmatrix}
    x' \\
    \dot{x}' \\
    \theta' \\
    \dot{\theta}' 
    \end{pmatrix}
    =
    f(x, \dot{x}, \theta, \dot{\theta}, a).
\end{equation}

In our stochastic setting, we introduce model uncertainty by injecting Gaussian noise into the pole angle after applying the action. Specifically, the next state becomes:
\begin{equation}
    \begin{pmatrix}
    x' \\
    \dot{x}' \\
    \theta' + \epsilon \\
    \dot{\theta}' 
    \end{pmatrix}
    =
    f(x, \dot{x}, \theta, \dot{\theta}, a)
    +
    \begin{pmatrix}
    0 \\
    0 \\
    \epsilon \\
    0
    \end{pmatrix},
    \quad \epsilon \sim \mathcal{N}(0, \sigma_\theta^2(x))
\end{equation}
This results in a stochastic transition model where noise is added only to the pole angle. The transition distribution is given by:
\begin{equation}
\begin{aligned}
&P(\tilde{x}, \dot{\tilde{x}}, \tilde{\theta}, \dot{\tilde{\theta}} \mid x, \dot{x}, \theta, \dot{\theta}, a) = \\
&\delta((\tilde{x}, \dot{\tilde{x}}, \dot{\tilde{\theta}}) - (x', \dot{x}', \dot{\theta}')) \mathcal{N}\left(
\tilde{\theta};
\theta',
\sigma_\theta^2(x)
\right)
\end{aligned}
\end{equation}

In regions of the state space considered safe, both the true and estimated models use the same noise level \( \sigma_{low}^2 \), and hence their total variation (TV) distance is zero. In contrast, in uncertain or “dangerous” regions, the true transition model uses a higher variance \( \sigma_{high}^2 \), while the agent underestimates the noise using \( \sigma_{low}^2 \). The total variation distance between these two distributions is:
\begin{align}
&\mathrm{TV}\bigl(\mathcal{N}(\mu, \sigma_{low}^2),\, \mathcal{N}(\mu, \sigma_{high}^2)\bigr) = \nonumber \\
&\operatorname{erf}\!\left(
\frac{\sigma_{high} \sqrt{\ln(\sigma_{high} / \sigma_{low})}}{\sqrt{2(\sigma_{high}^2 - \sigma_{low}^2)}}
\right)
-
\operatorname{erf}\!\left(
\frac{\sigma_{low} \sqrt{\ln(\sigma_{high} / \sigma_{low})}}{\sqrt{2(\sigma_{high}^2 - \sigma_{low}^2)}}
\right)
\end{align}

\begin{figure}[ht]
    \centering
    \includegraphics[width=0.8\linewidth]{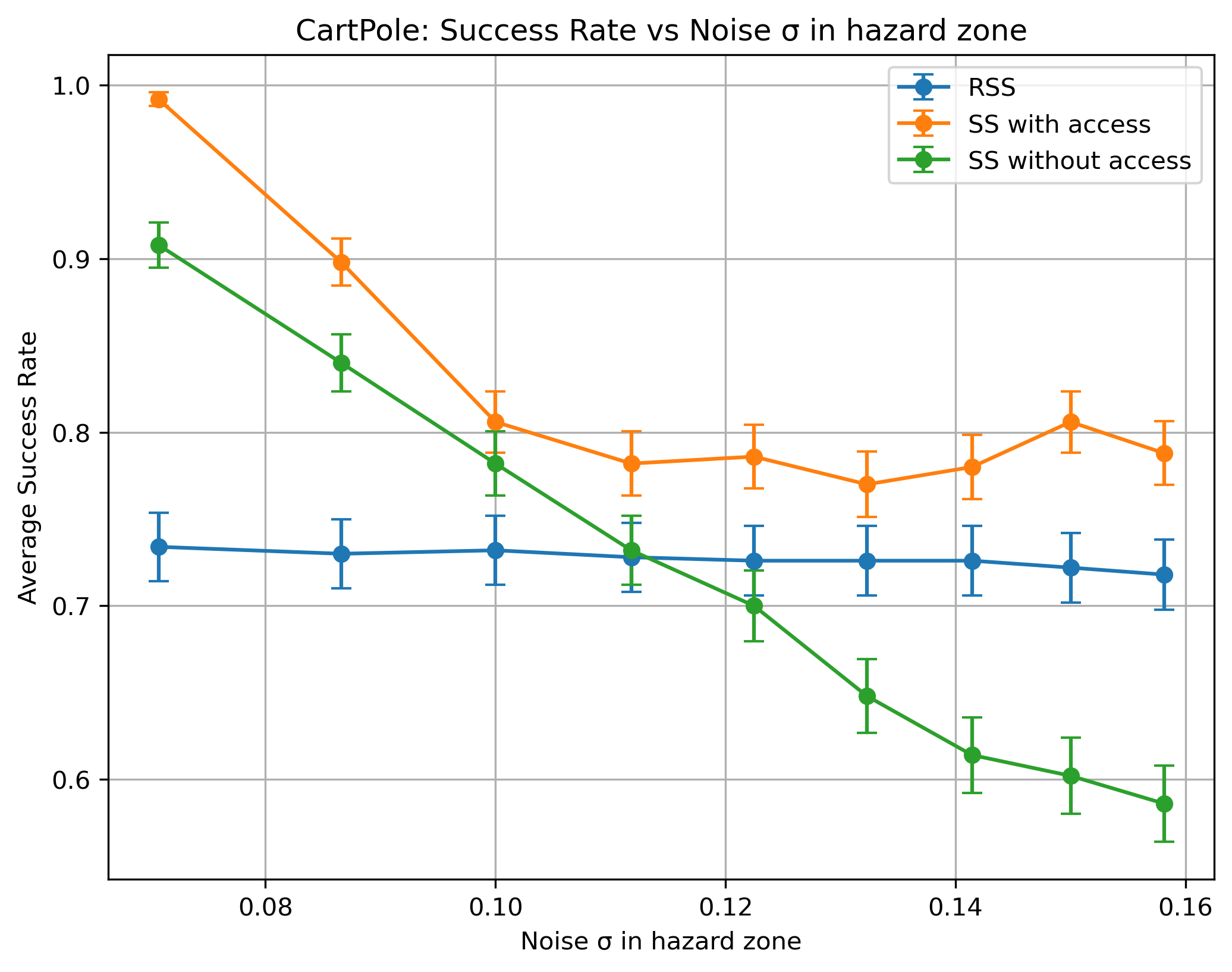}
    \caption{Average success rate (i.e., completing 200 steps without termination) comparison of RSS and SS (with/without access to true dynamics) under increasing noise variance in the hazard zone. Error bars denote standard error across 500 different seeds.}
    \label{fig:cartpole_avg_success_vs_noise}
\end{figure}
}{
}

\end{document}